\documentclass[runningheads]{llncs}

\usepackage[T1]{fontenc}
\usepackage{graphicx}
\usepackage{mathtools}  
\usepackage{dsfont}
\usepackage{amsmath}
\usepackage{mismath}
\usepackage{pgfplots}
\pgfplotsset{compat=1.7}
\usepackage{amssymb}
\usepackage{cite}

\usepackage[font=footnotesize]{caption}

\newcommand{\qedherenew}{\quad\Box}
\newenvironment{proofwoqed}{\noindent {\textbf{Proof:}}}{}

\usepackage[ruled,vlined,linesnumbered]{algorithm2e}
\SetKwComment{Comment}{/* }{ */}
\SetKwInOut{Input}{Input}
\SetKwInOut{Output}{Output}

\newcommand{\indic}[1]{\mathds{1}\{#1\}}

\newcommand{\ie}{i.\,e.\xspace}
\newcommand{\ronemax}{$r$\nobreakdash-OneMax\xspace}
\newcommand{\gonemax}{\textit{G}\nobreakdash-OneMax\xspace}
\newcommand{\rcga}{$r$\nobreakdash-cGA\xspace}
\begin{document}

\title{Runtime Analysis of a Multi-Valued Compact Genetic Algorithm on Generalized \text{OneMax}\thanks{This work was supported by a grant by the Danish Council for Independent Research (10.46540/2032-00101B).}}
\titlerunning{Runtime Analysis of \rcga on Generalized \text{OneMax}}
%
\author{Sumit Adak \and
Carsten Witt}
\authorrunning{S. Adak and C. Witt}
%
\institute{DTU Compute, Technical University of Denmark\\ Kgs. Lyngby, Denmark\\
\email{\{suad,cawi\}@dtu.dk}}
\maketitle              
\begin{abstract}

A class of metaheuristic techniques called estimation-of-distr\-ibution algorithms (EDAs) are employed in optimization as more sophisticated substitutes for traditional strategies like evolutionary algorithms. EDAs generally drive the search for the optimum by creating explicit probabilistic models of potential candidate solutions through repeated sampling and selection from the underlying search space.

Most theoretical research on EDAs has focused on pseudo-Boolean optimization. Jedidia et al.\ (GECCO 2023) proposed the first EDAs for optimizing problems involving multi-valued decision variables. By building a framework, they have analyzed the runtime of a multi-valued UMDA on the $r$-valued \textsc{LeadingOnes} function. Using their framework, here we focus on the multi-valued compact genetic algorithm (\rcga) and provide a first runtime analysis of a generalized \text{OneMax} function.

To prove our results, we investigate the effect of genetic drift and progress of the probabilistic model towards the optimum.
After finding the right algorithm parameters, we prove that the \rcga solves this $r$-valued \text{OneMax} problem efficiently. We show that with high probability, the runtime bound is $\bigo(r^2 n \log^2 r \log^3 n)$. At the end of experiments, we state one conjecture related to the expected runtime of another variant of multi-valued \text{OneMax} function.

\end{abstract}

\section{Introduction}
\label{section:intro}

The term estimation-of-distribution algorithms (EDAs) refers to a type of algorithm that creates a probabilistic model that is subsequently utilized to produce new search points based on past searches. The fundamental distinction from evolutionary algorithms (EAs) are that they evolve a probabilistic model rather than a population. The performance of EDAs can be more efficient than EAs 
\cite{benbaki2021rigorous,doerr2021runtime,hasenohrl2018runtime,witt2023majority}. EDAs are constructed by carrying out three basic steps: first, a population of individuals is sampled using the current probabilistic model; second, the fitness of this population is determined; and finally, a new probabilistic model is generated based on the fitness of this population and the probabilistic model.

In this framework, various probabilistic models and updating methodologies provide distinct algorithms. The probabilistic model in multivariate EDAs includes inter-variable relationships. Some popular examples of multivariate EDAs include mutual-information-maximization input clustering (MIMIC) \cite{de1996mimic}, bivariate marginal distribution algorithm (BMDA) \cite{pelikan1999bivariate}, extended compact genetic algorithm (ecGA) \cite{harik2006linkage}, and many more. Univariate EDAs are another sort of EDA in which the probabilistic model's positions are mutually independent. Some examples of univariate EDAs include population-based incremental learning (PBIL) \cite{baluja1994population}, univariate marginal distribution algorithm (UMDA) \cite{muhlenbein1996recombination}, compact genetic algorithm (cGA) \cite{harik1999compact}, and others as well. Because the dependencies in multivariate EDAs are difficult to analyze mathematically, the majority of EDAs theoretical studies focus on univariate models \cite{Krejca2020}. This manuscript likewise focuses on univariate EDAs.

In conventional genetic algorithms on bit strings, the frequency of bit values in the population are controlled by the bit's contribution to fitness as well as random fluctuations caused by other bits with a higher influence on fitness. Random fluctuations can even cause particular bits to converge to a single value that differs from the best solution. This effect is precisely known as \textit{genetic drift} \cite{kimura1964diffusion,asoh1994mean}. Genetic drift is also observed in EDAs. According to Krejca and Witt \cite{Krejca2020}, genetic drift in EDAs is a broad concept of martingales, which are random processes with zero expected change that finally may stop at absorbing bounds of the underlying interval. Furthermore, Witt \cite{witt2019upper} and Lengler et al.\ \cite{lengler2018medium} demonstrated that, depending on the parameter, genetic drift might result in a significant performance loss on the \text{OneMax} function.

Classical evolutionary algorithms are commonly employed for a variety of search spaces, whereas EDAs have traditionally been used for issues involving binary decision variables. In \cite{Jedidia-GECCO-2023}, Jedidia et al.\ take the initial steps toward applying EDAs for situations involving decision variables with more than two values. They specify univariate EDAs for multivariable decision factors. They treat a multi-valued problem by introducing $r$ probability values for each variables. Particularly, by building a framework, they analyzed the runtime of multi-valued UMDA on the $r$-valued \textsc{LeadingOnes} function. This article also deals with multi-valued EDAs by using their framework.  

In this paper, we study the $r$-valued compact genetic algorithm (\rcga) as essentially defined in \cite{Jedidia-GECCO-2023} and provide a first runtime analysis of a generalized \text{OneMax} function. In their framework \cite{Jedidia-GECCO-2023}, the authors have used the frequencies with some specific border restriction, but in our work there is no such border restriction. We perform a mathematical runtime of the \rcga on the $r$-valued \text{OneMax} function. Here, we study the \ronemax function. We bound the runtime with high probability. In our final analysis, we establish that the \rcga can optimize the \ronemax problem in $\bigo(r^2 n\log^2 r \log^3 n)$ (Theorem~\ref{theorem:rcgacomplexity}). Further, we compare our bound with the bound of binary cGA and mention that our bound probably is not tight. At the end of this work, we state one conjecture related to the expected runtime of another variant of $r$-valued \text{OneMax}, which is named \gonemax. 

The article is organized as follows: Section~\ref{section:RW} describes the previous works related to our technical topics. The following section establishes the terminologies and defines the multi-valued \text{OneMax} function. Section~\ref{section:framework} elaborates on the multi-valued EDA framework for \rcga. Sections~\ref{section:drift} and~\ref{section:runtime} include the major technical results, genetic drift analysis, and the runtime evaluation of \rcga on the $r$-valued \text{OneMax}. The experiments in Section~\ref{section:experiments} demonstrate the average runtime throughout the entire parameter range for the hypothetical population size ($K$). Finally, the manuscript concludes with a brief summary.

\section{Related Work}
\label{section:RW}

This work is separated into three technical topics: the first one is the framework for EDAs, the next one concerns genetic drift, and finally a runtime evaluation on the multi-valued \text{OneMax} function. There are numerous theoretical papers on traditional evolutionary algorithms for multi-valued decision variables. 

Model-based optimization approaches have enabled EDAs to tackle a wide range of large and complicated problems \cite{droste2006rigorous,DROSTE200251,sudholt2019choice}. Droste conducted the first rigorous runtime analysis of an EDA \cite{droste2006rigorous}. For a simple EDA, the compact genetic algorithm (cGA) was used with the \text{OneMax} function. They provide an overall lower and upper bounds for all functions. It was also noticed that EDAs optimize problems in an entirely different way, as seen by the difference in runtime given on two linear functions, which contrasts the well-known analysis of how the (1+1)~EA optimizes linear functions by Droste et al.\ \cite{DROSTE200251}. 

Most theoretical research on EDAs has focused on pseudo-Boolean optimization \cite{dang2015simplified,friedrich2016edas,krejca2017lower}. Jedidia et al.\ \cite{Jedidia-GECCO-2023} recently introduced the EDAs for optimizing problems involving more than two decision variables from the domain $\{0,\dots,r-1\}^n$. They prove that the multi-valued EDAs solve the $r$-valued \textsc{LeadingOnes} problem efficiently. Overall, they demonstrate how EDAs can be tailored to multi-valued issues and used to assist define their parameters. At present, there is a very active research area analyzing EDAs on complex problems. Further, in \cite{Krejca2020,larranaga2001estimation,pelikan2015estimation}, one can find out more details about theory and practice of EDAs.

\section{Preliminaries}
\label{section:preliminaries}

We are considering the $r$-valued compact genetic algorithm (\rcga) introduced by \cite{Jedidia-GECCO-2023} to maximize an $r$-valued \text{OneMax} function. Here, we're looking at the maximization of functions of the kind $f\colon \{0,1,\dots,r-1\}^{n}\rightarrow \mathbb{R}$, which is named $r$-valued fitness functions. We define $f(x)$ as the \emph{fitness} of $x$, where $x \in \{0,1,\dots,r-1\}^{n}$ is an individual.

In this work, we discuss two different types of multi-valued \text{OneMax} functions: one is \ronemax and another is \gonemax. Let $n\in \mathbb{N}_{\geq 1}$ and $r\in \mathbb{N}_{\geq 2}$. In the following, we give the definition of \ronemax and \gonemax, where, for all $x = (x_1, \dots, x_n) \in \{0,1,\dots,r-1\}^{n}$,
\begin{align*}
r\text{-OneMax}(x) \coloneqq  \sum_{i=1}^{n} \indic {x_{i} = r-1} 
\text{\quad and\quad}
G\text{-OneMax}(x) \coloneqq \sum_{i=1}^{n} x_i
\end{align*}

In both functions, the single maximum is the string all-$(r-1)$s. However, a more general variant can be defined by selecting a random optimum $a\in \{0,\dots,r-1\}^{n}$ and defining, for all $b\in \{0,\dots,r-1\}^{n}$, $r$-$\text{OneMax}_{a}(b) = n - d(b,a)$, where $d(b,a) = \sum_{i=1}^{n} \indic{b_i\neq a_i}$ indicates the distance between the two strings $a$ and $b$. In the same way, we can define the generalized \gonemax. For an arbitrary optimum $a\in \{0,\dots,r-1\}^{n}$, and defining, for all $b\in \{0,\dots,r-1\}^{n}$, \textit{G}-$\text{OneMax}_{a}(b) = n\cdot (r-1) - d (b,a)$ where $d (b,a) = \sum_{i=1}^{n} \min\{\vert a_i - b_i\vert, r-\vert b_i - a_i\vert\} $ denotes the distance between the strings $a$ and $b$. The difference between \ronemax and \gonemax is \ronemax only distinguishes 
between whether value~$r-1$ is taken at a position or not, while \gonemax takes all $r$ values per position into account. The maximum fitness value for \ronemax is $n$, and for \gonemax it is $n(r-1)$.

A random variable $Z$ is said to \textit{stochastically dominate} another random variable $Y$, denoted by $Z\succeq Y$, if and only if for all $\lambda \in\mathbb{R}$ we have $\Pr[Z\leq \lambda]\leq \Pr[Y\leq \lambda]$.

\section{The Framework}
\label{section:framework}
In \cite{Jedidia-GECCO-2023}, Jedidia et al.\ proposed a framework for EDAs to optimize fitness functions with $r$-values. In this paper, we adopt their framework to characterize the underlying probabilistic model. Here, we concentrate on the $r$-valued cGA. 

\textbf{\rcga}: The Compact genetic algorithm (cGA) \cite{harik1999compact} is a widely used univariate EDA. The cGA has only one parameter, $K\in\mathbb{R}_{>0}$, which refers to the so-called hypothetical population size \cite{doerr2021runtime} and it maintains a vector of probabilities (called frequencies). In each iteration, two solutions are created independently. After computing the fitness value, it changes each frequency by $1/K$ so that the frequency of the better sample increases while the frequency of the worse sample decreases. The \rcga is an expanded variant of cGA that takes into account multiple variables. The \rcga, outlined in Algorithm~\ref{algorithm:r-cGA-rOneMax}, employs marginal probabilities (again denoted  as frequencies) $p^{(t)}_{i,j}$ that correspond to the probability of position $i$ and value $j$ at time $t$. In each iteration, the sampling distribution generates two solutions $x$ and $y$ independently. After that, the fitter offspring is determined among $x$ and $y$, and the frequencies are modified by a step amount of $\pm 1/K$ in the prospective direction of the better offspring for positions where both offspring vary. In this way, $K$ indicates the strength of the probabilistic model update. 

The probabilistic model for \rcga is an $n\times r$ matrix (the frequency matrix), with each row $i\in\{1,\dots,n\}$ forming a vector $p_i \coloneqq (p^{(t)}_{i,j})_{j\in \{0,\dots,r-1\}}$ (the frequency vector at position $i$). After the update, each frequency vector in the \rcga sums to 1, since one frequency is increased by $1/K$ and one frequency is decreased by the same quantity. In general, we are interested in how many function evaluations the \rcga performs before sampling the optimum. This quantity is also referred to as runtime or optimization time.

\vspace{-1em}
\begin{algorithm}[h]
\caption{$r$-valued Compact Genetic Algorithm ($r$-cGA) for the maximization of $f : \{0,\dots,r-1\}^n \rightarrow \mathbb{R}$}
\label{algorithm:r-cGA-rOneMax}
\KwData{$t \gets 0$ \hspace{18em} $p^{(t)}_{i,0} \gets p^{(t)}_{i,1} \gets  p^{(t)}_{i,2} \dots \gets p^{(t)}_{i,r-1} \gets \frac{1}{r}$ where $i\in \{1, 2, \dots, n\}$}
\While{termination criterion not met}{
\For{$i\in \{1, 2, \dots, n\}$}{
$x_{i} \gets j$ with probability $p^{(t)}_{i,j}$ w.r.t. $j=0,\dots,r-1$, independently for all $i$ \\
$y_{i} \gets j$ with probability $p^{(t)}_{i,j}$ w.r.t. $j=0,\dots,r-1$, independently for all $i$ \\
}
\If{$f(x) < f(y)$}{
\For{$i\in \{1, 2, \dots, n\}$}{
swap $x_i$ and $y_i$
}
}
\For{$i\in \{1, 2, \dots, n\}$}{
\For{$j\in \{0, 1, \dots, r-1\}$}{
$p^{(t+1)}_{i,j} \gets p^{(t)}_{i,j} + \frac{1}{K} (\indic{x_{i} = j} - \indic{y_{i}=j})$}
}
$t\gets t + 1$
}
\end{algorithm}\vspace{-1em}

\noindent \textbf{The probabilistic model}: This paragraph depicts the stochastic process in the algorithm. Let $p^{(t)}_{i,j}$ be the marginal probability at time $t$ for arbitrary position $i$ and value $j$ where $(i,j)\in \{1,\dots,n\}\times \{0,\dots,r-1\}$. An $r$-valued EDA's probabilistic model is an $n\times r$ matrix of $(p^{(t)}_{i,j})_{i,j}$ (the frequency matrix), with each row $i$ forming a frequency vector of probabilities that sum to 1. When constructing an individual $x\in \{0,\dots,r-1\}^{n}$, for all $i\in \{1,\dots,n\}$ and all $j \in \{0,\dots,r-1\}$, the probability that $x_i$ has value $j$ is $p^{(t)}_{i,j}$. For every $y\in \{0,\dots,r-1\}^{n}$, we can state that $\Pr[x=y] = \Pi_{i\in \{1,\dots,n\}} \Pi_{j\in \{0,\dots,r-1\}} (p^{(t)}_{i,j})^{\indic{y_{i}=j}}$, where we assume that $0^{0}=1$. Furthermore, the frequency matrix is initialized so that each frequency equals $1/r$, indicating a uniform distribution. After each iteration, one ensures that each row sum amounts to 1.

Note that, in this work, contrary to the model in \cite{Jedidia-GECCO-2023}, the marginal probabilities are not restricted to some specific intervals. In general, the lower and upper borders on frequencies are 0 and 1. We do not use the borders from~\cite{Jedidia-GECCO-2023}, the reason is that because of those borders the analysis will be much more complicated. 
We make the following \textit{well-behaved frequency assumption}: the \rcga of any two frequencies can vary by a factor of $1/K$. In the absence of borders, the \rcga can employ frequencies in $\{0,1/K,2/K,\dots,1\}$, where $1/r$ is a multiple of $1/K$.

Clearly, $p^{(t)}_{i,j}$ is a random variable, and its change of value in one step is defined as $\Delta_{i,j}\coloneqq \Delta^{t}_{i,j}\coloneqq p^{(t+1)}_{i,j} - p^{(t)}_{i,j}$. Therefore, we can write $\Delta_{i,r-1}\coloneqq p^{(t+1)}_{i,r-1} - p^{(t)}_{i,r-1}$. This change is determined by whether the value of position $i$ influences the decision to update with respect to the first string $x$ sampled at time $t$ or the second string $y$. Particularly, we inspect the changes in the \ronemax value at all positions except $i$. To achieve this, we define $D_{i} := \left( \sum_{j\neq i} \indic{x_{j}=r-1} - \sum_{j\neq i} \indic{y_{j}=r-1}\right)$.

At this point, \rcga experiences two different kinds of steps which we discuss below. The following analysis and the terms ``rw-steps'' and ``biased steps'' follow closely the one from \cite{sudholt2019choice} for the binary cGA.\\
\textit{Random-walk steps:} If $\vert D_{i}\vert \geq 2$, position $i$ has no impact on the decision to update with respect to string $x$ or $y$. With $\Delta_{i,r-1} \neq 0$, it is necessary that position $i$ for value $r-1$ is sampled differently. That means, the value of $p^{(t)}_{i,r-1}$ will be increased or decreased by $1/K$ with equal probability $p^{(t)}_{i,r-1} (1 - p^{(t)}_{i,r-1})$. Otherwise, based on the remaining probability, it holds $p^{(t+1)}_{i,r-1} = p^{(t)}_{i,r-1}$. Now, we can describe this by taking a variable $F_{i}$ where

\[F_{i}\coloneqq \begin{cases}
+1/K & \text{with probability } p^{(t)}_{i,r-1} (1 - p^{(t)}_{i,r-1}),\\
-1/K & \text{with probability }  p^{(t)}_{i,r-1} (1 - p^{(t)}_{i,r-1}),\\
0 & \text{with the remaining probability.}
\end{cases}\]

A step where $\vert D_{i}\vert \geq 2$ is referred to as a \textit{random-walk step (rw-step)} because the process is a fair random walk (along with self-loops) as $E(\Delta_{i,r-1}\mid p^{(t)}_{i,r-1}, \vert D_{i}\vert \geq 2) = \E(F_{i}\mid p^{(t)}_{i,r-1}) = 0$.

If $D_{i} =1$, then  $\left(\sum^{n}_{i=1} \indic{x_{i}=r-1}\right) \geq  \left(\sum^{n}_{i=1} \indic{y_{i}=r-1}\right)$ and for that strings $x$ and $y$ are never swapped in the \rcga. So, as previous, here we obtain the same argumentation. In addition, the process also executes a \textit{rw-step}.\\
\textit{Biased steps:} If $D_{i} =-1$, then the strings $x$ and $y$ are swapped unless position~$i$ is sampled as $x_{i}=r-1$ and $y_{i}\neq r-1$. As a result, both events of sampling position $i$ raise the $p^{(t)}_{i,r-1}$ value in different ways. So, we obtain $\Delta_{i,r-1} = 1/K$ with probability $2p^{(t)}_{i,r-1} (1 - p^{(t)}_{i,r-1})$ and $\Delta_{i,r-1} = 0$ else. 

If $D_{i} = 0$, then both the events of sampling position $i$ raise the $p^{(t)}_{i,r-1}$ value differently, similar to the previously examined situation ($D_{i} = -1$). And, again we have $\Delta_{i,r-1} = 1/K$ with  probability $2p^{(t)}_{i,r-1} (1 - p^{(t)}_{i,r-1})$ and $\Delta_{i,r-1} = 0$ otherwise. Let us take a random variable $B_{i}$ such that 

\[B_{i}\coloneqq \begin{cases}
+1/K & \text{with probability } 2p^{(t)}_{i,r-1} (1 - p^{(t)}_{i,r-1}),\\
0 & \text{with the remaining probability.}
\end{cases}\]

For $D_{i} =-1$ and $D_{i} = 0$, we can conclude that $\Delta_{i,r-1}$ follows the same distribution as $B_{i}$. A \textit{biased step (b-step)} occurs when $\E(\Delta_{i,r-1}\mid p^{(t)}_{i,r-1}, D_{i}\in\{-1,0\})=$ $\E(B_{i}\mid p^{(t)}_{i,r-1})=$ $2p^{(t)}_{i,r-1}(1-p^{(t)}_{i,r-1})/K >0$. 

The event whether a step is a \textit{rw-step} or a \textit{b-step} for position~$i$ is merely depending on external factors that are stochastically independent of the outcome of position~$i$. Let $R_{i}$ represent the occurrence where $D_{i} =1$ or $\vert D_{i}\vert \geq 2$. We arrive at the following equality: 
\begin{equation}
\label{equation:calDelta}    
\Delta_{i,r-1} = F_{i}\cdot \indic{R_i} + B_{i}\cdot \indic{\overline{R_i}} 
\end{equation}
which we call \textit{superposition}. Informally, the change in $p^{(t)}_{i,r-1}$ value is a superposition of a unbiased random walk and biased steps. 

\section{Genetic Drift for the $r$-valued cGA}
\label{section:drift}

We show an upper bound on the influence of genetic drift for $r$-valued EDAs, similar to \cite{Doerr2020ITEV,Jedidia-GECCO-2023}. This enables us to select parameter values for EDAs that avoid the often undesirable effect of genetic drift. This section provides a general overview of genetic drift, followed by a concentration result for neutral positions. Finally, there is an upper bound for positions with \textit{weak preference}. Some proofs were not included in this paper due to limitations on space. They are included in the Appendix.

In EDAs, genetic drift occurs when a frequency does not reach extreme values 1 or 0 as a result of a clear signal from the objective function, but rather as a result of random fluctuations caused by the process's stochasticity. Researchers have explored genetic drift in EDAs explicitly \cite{shapiro2002sensitivity,shapiro2005drift,shapiro2006diversity} and conducted numerous runtime analyses \cite{witt2018domino,witt2019upper,doerr2020univariate,sudholt2019choice,droste2005not,lengler2021complex}. We analyze genetic drift for multi-valued EDAs, specifically the \rcga, based on insights from \cite{Jedidia-GECCO-2023} and the framework from~\cite{Doerr2020ITEV}.

Genetic drift is typically examined using a fitness function in the \textit{neutral} position. Let $f$ denote an $r$-valued fitness function. A position $i\in \{1,\dots,n\}$ is called \textit{neutral} (in relation to $f$), if and only if $x_i$ has no effect on the value of $f$ for all $x\in \{0,\dots,r-1\}^{n}$. That is, for all $x,x'\in \{0,\dots,r-1\}^{n}$ such that $x_{j}=x'_{j}$ for all $j\in \{1,\dots,n\} \setminus\{i\}$, we have $f(x)=f(x')$.

The frequencies of neutral variables in traditional EDAs without margins create martingales, which is useful for analyzing genetic drift \cite{Doerr2020ITEV}. This finding applies to EDAs with binary representation. Furthermore, the concept can be carried over to $r$-UMDA \cite{Jedidia-GECCO-2023}, too. We make this argument specific to the \rcga. 

\begin{lemma}
\label{lemma:neutral-freq-martingale} 
Let $f$ be an $r$-valued position and $i\in\{1,\dots,n\}$ be a neutral position of $f$. Consider the \rcga without margins optimizing $f$. Then, for each $j\in\{0,\dots, r-1\}$, the frequencies $(p^{(t)}_{i,j})_{t\in\mathbb{N}}$ are a martingale.
\end{lemma}

\begin{proof}
Since the algorithm has no margins, there are no restrictions in each iteration $t\in\mathbb{N}$, hence it holds that ${p}^{(t+1)}_{i,j} = p^{(t)}_{i,j} + \frac{1}{K} (\indic{x_{i} = j} - \indic{y_{i}=j})$. Since $i$ is neutral, the values at position $i$ do not affect the decision of updating the frequency with respect to $x$ and $y$. Particularly, updating of frequency depends on the three cases with respect to $x_i$ and $y_i$ ($x_i=j$ and $y_i\neq j$ constitutes the first case, $x_i\neq j$ and $y_i=j$ represent the second case, and all the remainder belongs to the third case): 
\[p^{(t+1)}_{i,j}\coloneqq \begin{cases}
p^{(t)}_{i,j}+1/K & \text{with probability }  p^{(t)}_{i,j} (1 - p^{(t)}_{i,j}),\\
p^{(t)}_{i,j}-1/K & \text{with probability }  p^{(t)}_{i,j} (1 - p^{(t)}_{i,j}),\\
p^{(t)}_{i,j} & \text{with probability }  1 - 2p^{(t)}_{i,j} (1 - p^{(t)}_{i,j}).
\end{cases}\]
By combining the three cases, we get
\begin{align*}
 & \E(p^{(t+1)}_{i,j}\mid p^{(t)}_{i,j}) \\
& = (p^{(t)}_{i,j}+1/K)\cdot p^{(t)}_{i,j}(1-p^{(t)}_{i,j}) + (p^{(t)}_{i,j}-1/K)\cdot p^{(t)}_{i,j}(1-p^{(t)}_{i,j}) + \\ & \indent p^{(t)}_{i,j}\cdot (1-2p^{(t)}_{i,j}(1-p^{(t)}_{i,j}))\\
& =  p^{(t)}_{i,j},
\end{align*}
proving the claim.
\qed\end{proof}

In \cite{Jedidia-GECCO-2023}, all frequencies of an EDA start at a value $1/r$ and analyses for smaller deviations in both direction up to $1/(2r)$. In this article, for \rcga we follow the same frequency setting starting from $1/r$ and tolerate up to $1/(2r)$ in either direction.

We apply a martingale concentration result \cite[Theorem 3.13]{mcdiarmid1998concentration} to exploit the lower sampling variance at frequencies in $\Theta (1/r)$. And, we restate an adjusted version of a theorem by McDiarmid \cite[eq. (41)]{mcdiarmid1998concentration} that was used by Doerr and Zheng \cite{Doerr2020ITEV} and by Jedidia et al.\ \cite{Jedidia-GECCO-2023}.

\begin{theorem}
\label{theorem:Hoeffding-Azuma-inequality}   
Let $a_1,\dots,a_m\in \mathbb{R}$, and $X_1,\dots,X_m$ be a martingale difference sequence with $\lvert X_k \rvert\leq a_k$ for each $k$. Then for all $\varepsilon \in \mathbb{R}_{\geq 0}$, it holds that
\[\Pr\left [\max_{k=1,\dots,m} \left\vert \sum_{i=1}^{k} X_{i}  \right\vert \geq \varepsilon\right ] \leq 2\exp{\left (-\frac{\varepsilon^{2}}{2\sum_{i=1}^{m} a^{2}_{i}}\right )}.\]
\end{theorem}

Next, we use Theorem~\ref{theorem:Hoeffding-Azuma-inequality} to demonstrate how long the frequencies of the \rcga at neutral places remain concentrated around the starting value of $p_{i,j}^{(0)}$ (which is usally $1/r$).

\begin{theorem}
\label{theorem:neutral-frequency-stay}    
Let $f$ be an $r$-valued fitness function with a neutral position $i\in\{1,\dots,n\}$. Consider the \rcga optimizing $f$ with population size $K$. Then, for $j\in \{0,\dots,r-1\}$ and $T\in\mathbb{N}$, we have
\[\Pr\left [\max_{t\in\{0,\dots,T\}} \left\vert p^{(t)}_{i,j} - p_{i,j}^{(0)}  \right\vert \geq \frac{1}{2r}\right ] \leq 2\exp{\left (-\frac{K^2}{8Tr^2}\right )}.\]
\end{theorem}

\begin{proof}
We follow the same proof method as in the proof of \cite[Theorem 2]{Doerr2020ITEV}. For the \rcga, we have a sequence of frequencies $(p^{(t)}_{i,j})_{t\in\mathbb{N}}$. For all $j\in \{0,\dots, r-1\}$, we obtain
\begin{align*}
\P[p^{(t+1)}_{i,j} = p^{(t)}_{i,j}+1/K \mid p^{(1)}_{i,j},\dots, p^{(t)}_{i,j}] & = p^{(t)}_{i,j} (1 - p^{(t)}_{i,j})\\
\P[p^{(t+1)}_{i,j} = p^{(t)}_{i,j}-1/K \mid p^{(1)}_{i,j},\dots, p^{(t)}_{i,j}] & = p^{(t)}_{i,j} (1 - p^{(t)}_{i,j}) \\
\P[p^{(t+1)}_{i,j} = p^{(t)}_{i,j} \mid p^{(1)}_{i,j},\dots, p^{(t)}_{i,j}] & = 1 - 2p^{(t)}_{i,j} (1 - p^{(t)}_{i,j})
\end{align*}
From  Lemma~\ref{lemma:neutral-freq-martingale}, we have $\E(p^{(t+1)}_{i,j} \mid p^{(0)}_{i,j},\dots, p^{(t)}_{i,j})=p^{(t)}_{i,j}$. Consider the martingale difference sequence $R_{t} \coloneqq p^{(t)}_{i,j} - p^{(t-1)}_{i,j}$ where $t \geq 1$ and $p^{(0)}_{i,j} = 1/r$, which satisfies $\lvert R_{t} \rvert \leq 1/K$. Further, by expanding it
\begin{align*}
p_{i,j}^{(k)} & = p_{i,j}^{(0)} + (p_{i,j}^{(1)}-p_{i,j}^{(0)}) + (p_{i,j}^{(2)}-p_{i,j}^{(1)}) + \dots + \\ & \indent (p_{i,j}^{(k)}-p_{i,j}^{(k-1)}) \\
& = p_{i,j}^{(0)} + R_{1} + \dots + R_{k}\\
& = p_{i,j}^{(0)} + \sum_{\ell=1}^{k} R_{\ell}\\
\text{so, } p_{i,j}^{(k)} - \frac{1}{r} & = \sum_{\ell=1}^{k} R_{\ell}
\end{align*}

According to \textit{Hoeffding-Azuma inequality} (Theorem~\ref{theorem:Hoeffding-Azuma-inequality}), we have 
\begin{align*}
\Pr\left [\max_{k=1,\dots,T} \left\vert p^{(k)}_{i,j} - \frac{1}{r}  \right\vert \geq \frac{1}{2r} \right] & \\
\Pr\left [\max_{k=1,\dots,T} \left\vert \sum_{\ell=1}^{k} R_{\ell} \right\vert \geq \frac{1}{2r} \right] & \leq 2\exp{\left (-\frac{K^2}{8Tr^2}\right )}.\quad\qed
\end{align*}
\end{proof}

In many situations, for a given fitness function the positions are not neutral. However, we demonstrate that the results on neutral positions apply to positions in which one value is better than all other values. This is known as \textit{weak preference}~\cite{Doerr2020ITEV}. Formally, an $r$-valued fitness function $f$ has a weak preference for a value $j\in \{0,\dots,r-1\}$ at a position $i\in \{1,\dots,n\}$ if and only if, for all $x_1,\dots,x_n \in \{0,\dots,r-1\}$, it holds that
\[ f(x_1,\dots,x_{i-1},x_{i},x_{i+1},\dots,x_n) \leq f(x_1,\dots,x_{i-1},j,x_{i+1},\dots,x_n).\]

We apply Lemma 3 by Doerr and Zheng~\cite[Lemma 3]{Doerr2020ITEV} according to the \rcga.

\begin{theorem}
\label{theorem:weak-preference-neutral}
Consider two $r$-valued fitness functions $f$ and $g$ to optimize using the \rcga, such that the first position of $f$ weakly prefers $r-1$ and the first position of $g$ is neutral.

Let $p$ and $q$ be the corresponding frequency matrices of $f$ and $g$, both defined by the \rcga. Then, for all $t\in\mathbb{N}$, it holds that $p^{(t)}_{1,r-1}\succeq q^{(t)}_{1,r-1}$.
\end{theorem}

\begin{proof}
We first demonstrate the claim for the first iteration, and then argue that an easy induction proves it for any iteration $t$. In the frequency matrices, we assume that $p^{(0)}_{i,r-1}= q^{(0)}_{i,r-1}$ for $i\in \{2,\dots,n\}$ and $p^{(0)}_{1,r-1}\geq q^{(0)}_{1,r-1}$. Assume that in the case of \rcga, the well-behaved frequency assumption holds. Note that, with the well-behaved frequency assumption, we have $p^{(0)}_{1,r-1} \geq q^{(0)}_{1,r-1}$ implies $p^{(0)}_{1,r-1} = q^{(0)}_{1,r-1}$ or $p^{(0)}_{1,r-1} = q^{(0)}_{1,r-1} + 1/K$. We only regard  the latter, more interesting case. We show $p^{(1)}_{1,r-1}\succeq q^{(1)}_{1,r-1}$ using the definition of domination, that is, that $\Pr[p^{(1)}_{1,r-1}\leq \lambda] \leq \Pr[q^{(1)}_{1,r-1}\leq \lambda]$ holds for all $\lambda\in\mathbb{R}$. Then, we have following three cases:
\begin{enumerate}
    \item At first, assume $\lambda < q^{(0)}_{1,r-1}$. Since $p^{(0)}_{1,r-1} - 1/K \geq q^{(0)}_{1,r-1} > \lambda$ from our assumption, we have $\Pr[p^{(1)}_{1,r-1}\leq\lambda] = 0 \leq \Pr[q^{(1)}_{1,r-1}\leq \lambda]$.
    \item Assume $q^{(0)}_{1,r-1}\leq \lambda < p^{(0)}_{1,r-1}$. In this case, $\Pr[p^{(1)}_{1,r-1}\leq \lambda] \leq p^{(0)}_{1,r-1}(1-p^{(0)}_{1,r-1})\leq 1/(2r)$ and $\Pr[q^{(1)}_{1,r-1}\leq \lambda] = 1 - q^{(0)}_{1,r-1}(1-q^{(0)}_{1,r-1})\geq 1 - 1/(2r)$, which gives the claim.
    \item Assume $\lambda\geq p^{(0)}_{1,r-1}$. Since $q^{(0)}_{1,r-1} + 1/K \leq p^{(0)}_{1,r-1}\leq \lambda$ from our assumption, we have $\Pr[q^{(1)}_{1,r-1}\leq \lambda] = 1 \geq  \Pr[p^{(1)}_{1,r-1}\leq \lambda]$.
\end{enumerate}
Hence, we have $p^{(1)}_{1,r-1}\succeq q^{(1)}_{1,r-1}$.

In order to extend the proof to arbitrary generation $t$, note that if we have $p^{(t-1)}_{1,r-1}\succeq q^{(t-1)}_{1,r-1}$, then we can find a coupling of the two probability spaces (see, \cite[Theorem 12]{Doerr2019TCS}) describing the states of the algorithm at the start iteration $t-1$ in such a way that for any point $\omega$ in the coupling probability space we have $p^{(t-1)}_{1,r-1}\geq q^{(t-1)}_{1,r-1}$. Conditional on the $\omega$ and by using the above argument for one iteration, we obtain $p^{(t)}_{1,r-1}\succeq q^{(t)}_{1,r-1}$. This implies that we also have $p^{(t)}_{1,r-1}\succeq q^{(t)}_{1,r-1}$ without conditioning on an $\omega$. \qed
\end{proof}

Applying Theorem~\ref{theorem:weak-preference-neutral} allows us to extend Theorem~\ref{theorem:neutral-frequency-stay} to positions with weak preference. Since 
their expected value may raise over time (formally, they are a submartingale), we 
state the deviation with respect to an arbitrary starting value 

\begin{theorem}
\label{theorem:frequency-weak-preference}  
Let $f$ be an $r$-valued fitness function with a weak preference for $r-1$ at position $i\in \{1,\dots,n\}$. Consider the \rcga optimizing $f$ with parameter~$K$. Let $T\in \mathbb{N}$, then we have 
\[\Pr\left [\min_{t\in\{0,\dots,T\}} p^{(t)}_{i,r-1} \leq p_{i,r-1}^{(0)}-\frac{1}{2r}\right ] \leq \mathord{2\exp}\mathord{{\left (-\frac{K^2}{8Tr^2}\right )}}.\]
\end{theorem}

\begin{proof}
Let $g$ be an $r$-valued fitness function with neutral position $i\in\{1,\dots,n\}$ and frequency matrix $q$. Consider the \rcga optimizing $g$. According to Theorem~\ref{theorem:weak-preference-neutral}, it follows for all $t\in\mathbb{N}$ that $p^{(t)}_{i,r-1}$ stochastically dominates $q^{(t)}_{i,r-1}$. By applying Theorem~\ref{theorem:neutral-frequency-stay} to fitness function $g$ for position $i$, we have
\[\Pr\left [\min_{t\in\{0,\dots,T\}} q^{(t)}_{i,r-1} \leq p_{i,r-1}^{(0)}-\frac{1}{2r}\right ] \leq \mathord{2\exp}\mathord{\left (-\frac{K^2}{8Tr^2}\right )}\]
Using the stochastic domination yields the tail bound for $f$. $\qedherenew$
\end{proof}

\section{Runtime Analysis}
\label{section:runtime}

This section evaluates the runtime of the \rcga (Algorithm~\ref{algorithm:r-cGA-rOneMax}) on \ronemax. In the preliminaries, we briefly presented the two variants of $r$-valued \text{OneMax} -- one is \ronemax and another is \gonemax. There is a single local maximum for both functions at the all-($r-1$)s string, which is also their global optimum.

With high probability, we bound the runtime of the \rcga on \ronemax under the assumption of low genetic drift using drift analysis and then apply Markov's inequality on the time bound. Further we consider the probability of no frequency dropping below $1/(2r)$ at the beginning and, over time, below $k/(2r)$ for a growing $k$. We prove the following theorem in a similar fashion as Sudholt and Witt \cite[Theorem 2]{sudholt2019choice} for binary decision variables; however, additional care has to be taken to control genetic drift from the starting value~$1/r$ of a frequency. In the following theorem (Theorem~\ref{theorem:rcgacomplexity}), we formulate our main results related to runtime. 

\begin{theorem}
\label{theorem:rcgacomplexity}
With high probability, the runtime of the \rcga on \ronemax with $K\geq c r^2 \sqrt{n} (\log r + \log^2 n) $ for a sufficiently large $c > 0$ and $K,r= poly(n)$ is $\bigo(K\sqrt{n}\log r\log n)$. For $K = c r^2 \sqrt{n} (\log r + \log n)$, the runtime bound is $\bigo(r^2 n \log^2 r \log^3 n )$.

\end{theorem}

The proof of Theorem~\ref{theorem:rcgacomplexity} requires the following lemmas. According to the lemmas, the drift grows with an update strength of $1/K$. Furthermore, a high value of $1/K$ also increases genetic drift. We prove the next lemma in a similar way as \cite[Lemma 1]{Neumann2010}.

\begin{lemma}
\label{lemma:positiveupdate}
Let $p^{(t)}_{i,j}$ denote the frequency vectors of the current iteration of \rcga on \ronemax where $(i,j)\in \{1,\dots,n\}\times \{0,\dots,r-1\}$. For a sufficiently large $n$, we get 
\[\P[D_{i}=0] \geq \frac{4}{9\left(2\sqrt{3\left( \sum_{j\neq i} p^{(t)}_{j,r-1} (1 - p^{(t)}_{j,r-1})\right)}+1\right)}\]
where $D_{i} := \left( \sum_{j\neq i} \indic{x_{j}=r-1} - \sum_{j\neq i} \indic{y_{j}=r-1}\right)$ and $x,y\in \{0,\dots,r-1\}^n$.
\end{lemma}

\begin{proof}
Let $x\in \{0,\dots,r-1\}^n $ and $y\in \{0,\dots,r-1\}^n$ be the next two \rcga solutions and $D_{i} \coloneqq X - Y$ where $X\coloneqq \sum_{j\neq i} \indic{x_j= r-1}$ and $Y\coloneqq \sum_{j\neq i} \indic{y_j= r-1}$, be the difference in the number of occurrence of $(r-1)$-entries for all other positions. Next, we estimate $\P[D_{i}=0]$.

The variance equals $\Var(\indic{x_{j}=r-1}) =$ $\E\left( (\indic{x_{j}=r-1} - p^{(t)}_{j,r-1})^2\right) = p^{(t)}_{j,r-1}(1-p^{(t)}_{j,r-1})$ and due to independence $\Var (X)=\sum_{j\neq i} \Var(\indic{x_{j}=r-1}) = \sum_{j\neq i} p^{(t)}_{j,r-1}(1-p^{(t)}_{j,r-1})$.

Let ${\sigma}^2 = \Var(X)$, then by Chebyshev's inequality
\begin{center}
\[\P\left[ \left\vert X - \sum_{j\neq i} p^{(t)}_{i,r-1} \right\vert \geq \sqrt{3}\sigma \right] \leq \frac{1}{3}.\]
\end{center}
Let $I := \left[ \sum_{j\neq i} p^{(t)}_{i,r-1} - \sqrt{3}\sigma, \sum_{j\neq i} p^{(t)}_{i,r-1} + \sqrt{3}\sigma  \right]$ and note that there are a maximum of $2\sqrt{3}\sigma +1$ integers in $I$. Assume that $X\in I$ and $Y\in I$, which occurs at least $(1-1/3)^2 = 4/9$ with probability. Note that, $X=Y$ is identical to $D_{i} = 0$ and 

\begin{align*}
\P&[X=Y \mid X,Y \in I]  \\
 & = \sum_{z\in I} \P[X=z \mid X\in I] \cdot \P[Y=z\mid Y\in I]  \\
 & = \sum_{z\in I} \P[X=z \mid X\in I]^2 \\
 & \geq \sum_{z\in I} {\left( \frac{1}{\vert I \vert}\right)}^2 = \frac{1}{\vert I \vert} \geq \frac{1}{2\sqrt{3}\sigma +1}
\end{align*} 
where the first inequality holds due to the square function's convexity (shifting probability mass to the average value $1/\vert I\vert$ minimizes the sum of squares). Therefore, the unconditional probability $\P[X=Y]=\P[D_{i}=0]$ is at least
\begin{align*}
\frac{4}{9}\cdot \frac{1}{2\sqrt{3}\sigma +1}
\end{align*}
Plugging in $\sigma$, it holds that
\[\P[D_{i}=0] \geq \frac{4}{9\left(2\sqrt{3\left( \sum_{j\neq i} p^{(t)}_{j,r-1} (1 - p^{(t)}_{j,r-1})\right)}+1\right)}. \qedherenew \]  
\end{proof}

\begin{lemma}
\label{lemma:sdrift}
 If $\frac{1}{K} \leq p^{(t)}_{i,r-1} \leq 1 - \frac{1}{K}$, then   
 \[\E(\Delta_{i,r-1}\mid p^{(t)}_{i,r-1}) \geq \frac{8(p^{(t)}_{i,r-1} (1 - p^{(t)}_{i,r-1}))}{9K \left(2\sqrt{3\left( \sum_{j\neq i} p^{(t)}_{j,r-1} (1 - p^{(t)}_{j,r-1})\right)}+1\right)}\]
\end{lemma}

\begin{proofwoqed}
We get the expected changes using Equation~\ref{equation:calDelta}, as 
\[\E(\Delta_{i,r-1}\mid p^{(t)}_{i,r-1}) = \E(F_{i}\mid p^{(t)}_{i,r-1})\cdot \indic{R_i} + \E(B_{i}\mid p^{(t)}_{i,r-1})\cdot \indic{\overline{R_i}}\]
From Section~\ref{section:framework}, we know $\E(F_{i}\mid p^{(t)}_{i,r-1}) = 0$ and $\E(B_{i}\mid p^{(t)}_{i,r-1}) = 2p^{(t)}_{i,r-1}(1-p^{(t)}_{i,r-1})/K$. Further, from Lemma~\ref{lemma:positiveupdate}, we got
\[\indic{\overline{R_i}}\geq \P[D_{i}=0] \geq \frac{4}{9\left(2\sqrt{3\left( \sum_{j\neq i} p^{(t)}_{j,r-1} (1 - p^{(t)}_{j,r-1})\right)}+1\right)}\]

By multiplying the results of $\E(B_{i}\mid p^{(t)}_{i,r-1})$ and $\P[D_{i}=0]$
 \[\E(\Delta_{i,r-1}\mid p^{(t)}_{i,r-1}) \geq \frac{8(p^{(t)}_{i,r-1} (1 - p^{(t)}_{i,r-1}))}{9K \left(2\sqrt{3\left( \sum_{j\neq i} p^{(t)}_{j,r-1} (1 - p^{(t)}_{j,r-1})\right)}+1\right)}. \qedherenew \]
\end{proofwoqed}

The following lemma accumulates the drift of single frequencies in  a potential 
function~$\varphi$ and will be crucial for the proof of Theorem~\ref{theorem:rcgacomplexity}. Its proof frequently uses 
the complementary frequencies  $q^{(t+1)}_{i,j} \coloneqq 1-p^{(t+1)}_{i,j}$.
\begin{lemma}
\label{lemma:drift-of-phi}
For any~$t\ge 0$, let $\varphi_t \coloneqq \sum^{n}_{i=1} 1 - p^{(t)}_{i,r-1} = n - \sum^{n}_{i=1} p^{(t)}_{i,r-1}$. If there is some $s>0$ such that 
for all $i\in\{1,\dots,n\}$ it holds that 
$p_{i,r-1}^{(t)} \ge s$ and furthermore $\varphi_t\ge 1/2$, then 
\[
\E(\varphi_t - \varphi_{t+1} \mid \varphi_t) \ge \frac{2s\sqrt{\varphi}}{15K}.
\]
\end{lemma}

\begin{proof}
We estimate the expectation of $\varphi'\coloneqq \varphi_{t+1} = \sum^{n}_{i=1} q^{(t+1)}_{i,r-1}$ based on $\varphi\coloneqq \varphi_t=\sum_{i=1}^n q_{i,r-1}^{(t+1)}$. First, 
we consider the drift of a single term $q_{i,r-1}^{(t)}$. If $p_{i,r-1}\leq 1-1/K$, then by Lemma~\ref{lemma:sdrift}
\[\E(q^{(t+1)}_{i,r-1}\mid q^{(t)}_{i,r-1}) \leq q^{(t)}_{i,r-1} - \frac{8(p^{(t)}_{i,r-1} (1 - p^{(t)}_{i,r-1}))}{9K \left(2\sqrt{3\left( \sum_{j\neq i} p^{(t)}_{j,r-1} (1 - p^{(t)}_{j,r-1})\right)}+1\right)}\]
We bound $p^{(t)}_{i,r-1}(1-p^{(t)}_{i,r-1})$ from below using our assumption $p^{(t)}_{i,r-1}\geq s$ 
and the sum from above using
\[\sum_{j\neq i} p^{(t)}_{i,r-1}(1-p^{(t)}_{i,r-1})\leq \sum^{n}_{j=i} (1-p^{(t)}_{i,r-1}) = \sum^{n}_{j=i} q^{(t)}_{i,r-1} = \varphi \]
Then, 
\begin{align*}
\E(q^{(t+1)}_{i,r-1}\mid q^{(t)}_{i,r-1}) & \leq q^{(t)}_{i,r-1} - \frac{8}{9}\cdot\frac{q^{(t)}_{i,r-1}}{K}\cdot s\cdot\left(\frac{1}{2\sqrt{3\varphi} +1}\right)\\
 & \leq q^{(t)}_{i,r-1} \left( 1-\frac{8s}{9K}\cdot \frac{1}{2\sqrt{3\varphi} +1} \right)
\end{align*}   

Putting all together,
\begin{align*}
\E(\varphi'\mid\varphi) & = \sum_{i=1}^{n} \E(q^{(t+1)}_{i,r-1})\mid q^{(t)}_{i,r-1}) 
  \leq  \sum_{i=1}^{n} q^{(t)}_{i,r-1} \left( 1-\frac{8s}{9rK}\cdot \frac{1}{2\sqrt{3\varphi} +1} \right)\\
 & \leq  \varphi  \left( 1-\frac{8s}{9K}\cdot \frac{1}{2\sqrt{3\varphi} +1} \right)
  \leq  \varphi - \frac{8s}{9K}\cdot \frac{\varphi}{2\sqrt{3\varphi} +1} \\
 &  \leq  \varphi - \frac{8s}{9K}\cdot \frac{{\varphi}^{1/2}}{2\sqrt{3\varphi} +1}\cdot {\varphi}^{1/2}
\end{align*}
Further, for $\varphi \geq 1/2$, the product of the first two fractions in the 
negative term can be bounded from below using 
\[ \frac{8s}{9K}\cdot \frac{{\varphi}^{1/2}}{2\sqrt{3\varphi} +1} 
\ge 
\frac{8s}{9K}\cdot \frac{{\varphi}^{1/2}}{2\sqrt{3\varphi} + 2\sqrt{3\varphi}} \ge \frac{2s}{15K} \]
By substituting this back, 
\[\E(\varphi'\mid\varphi) \leq \varphi - \frac{2s\varphi^{1/2}}{15K}\]
as suggested.
\qed\end{proof}


With these lemmas, we now provide the proof of the main statement.

\begin{proof}[Proof of Theorem~\ref{theorem:rcgacomplexity}]
By our assumptions on well-behaved frequencies, all frequencies are restricted to $\{0,1/K,1/2K,$ $\dots,1/r,\dots,1-1/K,1\}$.  
The main idea is to bound the expected optimization time under the assumption of low genetic drift using additive drift analysis in a sequence of certain phases 
and then variable drift analysis, using  
a potential function accumulating all frequencies for value~$r-1$. The drift 
of this potential has already 
been bounded in Lemma~\ref{lemma:drift-of-phi} above, using 
similar estimations as in \cite{sudholt2019choice}.
The aim is to 
 show that after $\bigo(K\sqrt{n}\log n\log r)$ iterations  the algorithm finds the global optimum, \ie, the string $(r-1)^n$, with high probability, 
if $K\ge cr^2\sqrt{n}(\log^2 n + \log r)$ for a sufficiently 
large constant~$c>0$.  

Let $p^{(t)}_{i,j}$ denote the  marginal probabilities at time~$t$ and $q^{(t)}_{i,j}\coloneqq 1 - p^{(t)}_{i,j}$ where $(i,j)\in \{1,\dots,n\}\times \{0,\dots,r-1\}$. Now, we use the potential function $\varphi_t = \sum^{n}_{i=1} q^{(t)}_{i,r-1}$, which calculates the distance to an ideal setting in which all frequencies for 
value~$r-1$ have reached their maximum. 
From the definition of $\varphi_t = \sum^{n}_{i=1} 1 - p^{(t)}_{i,r-1} = n - \sum^{n}_{i=1} p^{(t)}_{i,r-1}$, we can note when $\varphi$ falls, then the sum of the frequencies increases and therefore also the average frequency. This will allow us to use increasing bounds for 
$p^{(0)}_{i,r-1}$ in 
Theorem~\ref{theorem:frequency-weak-preference} as $\varphi$ falls.

Starting from initialization, we split the run of the \rcga into phases~$k=1,\dots,r/2$. 
Phase~$k$ starts at the first time where $\varphi_t\le n-kn/r$ 
and ends just before the first time where $\varphi_t\le n-(k+1)n/r$.  We will assume that for each $i\in\{1,\dots,n\}$, at the starting time~$T_k$ of phase~$k$ it holds 
that $p^{(T_k)}_{i,r-1}\ge k/(2r)$ and will analyze the probability 
of the bound holding below when studying genetic drift. 
Clearly, the bound is true at the starting 
time $T_1=0$ of phase~$1$.

Note that it is sufficient to reduce the potential by a total amount of 
$n/r$ to end any phase~$k$.  
From Lemma~\ref{lemma:drift-of-phi} we obtain a drift throughout phase~$k$ of 
\[\E(\varphi - \varphi'\mid\varphi) \ge \frac{2s\varphi^{1/2}}{15K} \ge \frac{k\sqrt{n/2}}{15Kr} \]
where we use our assumption $s\ge k/(2r)$ and bound $\varphi \ge n/2$ because $k\le r/2$, \ie, we only do the analysis until the potential has decreased to at most $n/2$. 

Now, we apply additive drift analysis with overshooting \cite[Theorem 4]{Doerr-MultiplicativeDrift}. 
An analysis of overshooting is necessary because at the point in time where the potential is at most $n-(k+1)n/r$, it does not have to be exactly $n-(k+1)n/r$ but  might be smaller. However, the target cannot be overshoot by much: even if all frequencies change by $1/K$, then the total change is at most $n/K$. And, by our assumption on~$K\ge cr^2(\log^2 n+\log r)$, we can simply pessimistically add the value $n/K\leq n/r$ (assuming~$c\ge 1$) to the actual distance $d=n/r$ to be overcome. As analyzed above, we have the drift bound $\delta = k\sqrt{n/2}/(15Kr)$ in phase~$k$ and the total distance is $(d+n/K)$. Then, the expected time to bridge the distance is $(d+n/K)/\delta \leq (2n/K)/\delta  = 2d/\delta$. So, we obtain the expected time to 
conclude phase~$k$ is at most  
\[ \frac{2d}{\delta} = \frac{2n}{r}\frac{15K}{\sqrt{n/2}}\frac{r}{k} = \bigo (K\sqrt{n}/k). \]
Applying Markov's inequality and a restart argument, the length 
of phase~$k$ is $\bigo (K\sqrt{n}\log n/k)$ with probability $1-\bigo(1/n^\kappa)$ 
for any constant~$\kappa>0$.
Further, summing over $k=1,\dots, r/2$, we obtain the total expected time spent 
in all phases is   
\[\bigo\left ( \sum_{k=1}^{r} \frac{K\sqrt{n}}{k}\right ) = \bigo \left ( K\sqrt{n} \sum_{k=1}^{r} \frac{1}{k}\right ) =  \bigo(K\sqrt{n}\log r).\]
In the same way, by adding up the tail bounds on the phase lengths, we have that 
the total time spent in phases $1,\dots,r/2$ is $\bigo(K\sqrt{n}\log r\log n)$ with probability at least  
$1-\bigO(rn^{-\kappa})$, using a union bound. Since $r=\mathit{poly}(n)$, this failure 
probability is still $o(1)$ if~$\kappa$ is a sufficiently large constant.

After the potential has decreased to at most~$n/2$, we can essentially use the the same analysis as for  the classical (binary) cGA on OneMax~\cite[Theorem 2]{sudholt2019choice}, with the exception that we do not use borders on frequencies 
here and, therefore, the potential $\varphi$ is non-increasing. We have $\varphi \leq n/2$ as starting point and assume $s\ge 1/2-1/(2r)\ge 1/4$. Using the variable drift theorem~\cite[Theorem 18]{sudholt2019choice}, we can estimate the expected time it takes for the potential function to drop from $\varphi\leq n/2$ to a value $\varphi \leq 1/2$. We choose the  drift function $h(\varphi)\coloneqq 2s\varphi^{1/2}/(15K)\ge \varphi^{1/2}/(30K)$,  for any $\varphi \geq 1/2$, in the 
variable drift theorem. Also similarly as in \cite[Proof of Theorem~2]{sudholt2019choice}, 
 we merge all the states with potentials $0 < \varphi < 1/2$ with state~0 so that 
the smallest state larger than 0 is $x_{\min}=1/2$. This modification can only increase the drift, hence the drift is still bounded from below by $h(\varphi)$ for all states $\varphi \geq x_{\min}$. Moreover, at $x_{\min} = 1/2$, the probability 
of sampling the optimum is at least~$1/2$ by the same arguments related to 
majorization and Schur-convexity as in \cite{sudholt2019choice}.

Now, the expected time to reach state~0 in the updated process, or any state with $\varphi < 1/2$ in the actual process, is at most
\[ \frac{1/2}{h(1/2)} + \int_{1/2}^{n} \frac{1}{h(\varphi)}\, d\varphi = \bigo(K) + \bigo(K) \cdot \int_{1/2}^{n} \varphi^{-1/2}\, d\varphi = \bigo(K\sqrt{n}). \]
Afterwards, after expected time at most~$2$, the optimum is sampled. 
Moreover, again by applying Markov's inequality and a restart argument, the 
time to sample the optimum is $\bigo(K\sqrt{n}\log n)$ with probability $1-o(1)$.

We still have to analyze the effect of genetic drift. In the analysis above, we assume that at the time 
$T_k$ beginning phase~$k$ (when the 
potential has become at most 
$n-nk/r$), every frequency $p_{i,r-1}^{(T_k)}$, where $i\in\{1,\dots,n\}$,  is bounded from below by $k/(2r)$. By definition, a potential of $\varphi_t$ corresponds 
to an average frequency value of $(n-\varphi_t)/n$. Moreover, since all $n$ frequencies $p_{i,r-1}^{(t)}$ describe the same stochastic process due to the 
symmetry of the \ronemax function, this average frequency equals the expected frequency at this time~$t$. Deviations of the frequency below the expected value can only happen in rw-steps. The aim is therefore to show 
that in each phase, the reduction 
of any frequency in rw-steps is bounded from above by $1/(2r)$. By Theorem~\ref{theorem:frequency-weak-preference} (using identifying its 
starting time~$0$ with $T_k$), the probability 
of a failure in a single phase, \ie, 
of a reduction more than $1/(2r)$ 
by genetic drift is at most 
$2\exp (K^2/(2Tr^2))$. Using 
$K\ge cr^2\sqrt{n}(\log^2 n+\log r)$ for 
some sufficiently large constant~$c>0$ and plugging 
in the above bound on the length of phase~$k$ of~$T=c'(K\sqrt{n}\log n)/k$ for another 
constant~$c'>0$, the failure probability is 
bounded from above by 
\[
 2\exp(K k /(2c'\sqrt{n}(\log n) r^2))
\le 2\exp(kc(\log n+\log r)/c') \le
2n^{-c/c'}r^{-c/c'}.
\]
Choosing $c$ large enough and 
taking a union bound over at most~$r$ phases and 
$n$ frequencies, the failure probability in this analysis of 
genetic drift is still 
$o(1)$. By choosing all constants appropriately, also the sum of all 
failure probabilities is $o(1)$.
\qed\end{proof}

\section{Experiments}
\label{section:experiments}
 
In the following section, we present the results of the experiments conducted to check the performance of the proposed algorithm without border restriction. Theoretically, we prove the expected time for \rcga on the \ronemax without margins. The algorithm is implemented in the $C$ programming language using the \text{WELL1024a} random number generator.

The experiments supplement our asymptotic analyses of the algorithm. We ran the \rcga on \ronemax for $n=500$ (Fig.~\ref{figure:rcGAn500}) without border restriction and all averaged over 3000 runs where $r\in\{3,4, \dots,10\}$. And, also for \gonemax, we ran the \rcga with same configuration. In all the cases, we observe the very same picture where the empirical runtime starts out from a very high value, takes a minimum and then increases again for the rest of the $K$. As an example, we got the minimum when $K$ is around 110 (Fig.~\ref{figure:rcGAn500}: Left-hand side and $r=4$). And after that it clearly goes up with $K$.

We can compare the results according to the variants. If we compare the plots for \ronemax and \gonemax, then both the variants produce the same structure. 
From the experimental setup, we can say that the bound of the theoretical analysis is not tight. We observed some empirical insights into the relationship between $K$ and runtime. After the close inspection, we can find the value of $K$ where the minimum of the runtime is reached. At the end of experiments, we state the following conjecture related to expected runtime of \rcga on \gonemax. The proof of the following conjecture is one of our future research subjects.


\begin{conjecture}
\label{conjecture:rcgacomplexity-gonemax}   
With high probability the expected runtime of \rcga on \gonemax with $K\geq cr\sqrt{n}\log r\log n$ for a sufficiently large $c>0$ and $K$, $r=poly(n)$ is $\bigo(K\sqrt{n}\log r)$. For $K = cr\sqrt{n}\log r\log n$ the runtime is $\bigo(r n \log^2 r\log n)$. 
\end{conjecture}

\begin{figure}
\vspace{-1em}
	\centering
\begin{tikzpicture}[scale=0.5]
\begin{axis}[
    xlabel={Population Size (K)},
    ylabel={Avg. Number of Iteration},
    xmin=50, xmax=1000,
    ymin=200, ymax=60000,
    legend pos= outer north east,
    legend cell align=left,
    ymajorgrids=true,
    grid style=dashed,
]

    
    \addplot[
    color=red,
    mark=dot,
    ]
    table[ignore chars={(,)},col sep=comma] {Data/rOneMax-n500-r3.txt}; 
    
    \addplot[
    color=green,
    mark=dot,
    ]
    table[ignore chars={(,)},col sep=comma] {Data/rOneMax-n500-r4.txt}; 
    
    \addplot[
    color=black,
    mark=dot,
    ]
    table[ignore chars={(,)},col sep=comma] {Data/rOneMax-n500-r5.txt}; 
    
    \addplot[
    color=blue,
    mark=dot,
    ]
    table[ignore chars={(,)},col sep=comma] {Data/rOneMax-n500-r6.txt}; 
    
    \addplot[
    color=magenta,
    mark=dot,
    ]
    table[ignore chars={(,)},col sep=comma] {Data/rOneMax-n500-r7.txt}; 
    
    \addplot[
    color=brown,
    mark=dot,
    ]
    table[ignore chars={(,)},col sep=comma] {Data/rOneMax-n500-r8.txt}; 
    
    \addplot[
    color=teal,
    mark=dot,
    ]
    table[ignore chars={(,)},col sep=comma] {Data/rOneMax-n500-r9.txt}; 
    
    \addplot[
    color=violet,
    mark=dot,
    ]
    table[ignore chars={(,)},col sep=comma] {Data/rOneMax-n500-r10.txt};

\legend{$r=3$, $r=4$, $r=5$, $r=6$, $r=7$, $r=8$, $r=9$, $r=10$}   
\end{axis}
\end{tikzpicture}%
\begin{tikzpicture}[scale=0.5]
\begin{axis}[
    xlabel={Population Size (K)},
    ylabel={Avg. Number of Iteration},
    xmin=50, xmax=1000,
    ymin=1000, ymax=80000,
    legend pos= outer north east,
    legend cell align=left,
    ymajorgrids=true,
    grid style=dashed,
]

    
    \addplot[
    color=red,
    mark=dot,
    ]
    table[ignore chars={(,)},col sep=comma] {Data/GOneMax-n500-r3.txt}; 
    
    \addplot[
    color=green,
    mark=dot,
    ]
    table[ignore chars={(,)},col sep=comma] {Data/GOneMax-n500-r4.txt}; 
    
    \addplot[
    color=black,
    mark=dot,
    ]
    table[ignore chars={(,)},col sep=comma] {Data/GOneMax-n500-r5.txt}; 
    
    \addplot[
    color=blue,
    mark=dot,
    ]
    table[ignore chars={(,)},col sep=comma] {Data/GOneMax-n500-r6.txt}; 
    
    \addplot[
    color=magenta,
    mark=dot,
    ]
    table[ignore chars={(,)},col sep=comma] {Data/GOneMax-n500-r7.txt}; 
    
    \addplot[
    color=brown,
    mark=dot,
    ]
    table[ignore chars={(,)},col sep=comma] {Data/GOneMax-n500-r8.txt}; 
    
    \addplot[
    color=teal,
    mark=dot,
    ]
    table[ignore chars={(,)},col sep=comma] {Data/GOneMax-n500-r9.txt}; 
    
    \addplot[
    color=violet,
    mark=dot,
    ]
    table[ignore chars={(,)},col sep=comma] {Data/GOneMax-n500-r10.txt};

\legend{$r=3$, $r=4$, $r=5$, $r=6$, $r=7$, $r=8$, $r=9$, $r=10$}    
\end{axis}
\end{tikzpicture}
\caption{\textmd{Left-hand side: empirical runtime of the \rcga on \ronemax, right-hand side: empirical runtime of the \rcga on \gonemax; for $n=500$, $K\in\{50,51,\dots,1000\}$ and averaged over 3000 runs.}}
\label{figure:rcGAn500}
\vspace{-3em}
\end{figure}

\section{Conclusion}
\label{section:comclusion}

We have performed a runtime analysis of a multi-valued EDA, namely the \rcga, on a generalized \text{OneMax} function. In our analysis, we have bound the runtime of the \rcga with a high probability. And, considering the increased complexity of the problem, the resulting runtime is understandable. Since \rcga is efficient on generalized OneMax, we believe that \rcga is a good algorithm for other, more complex problems, too.

A theoretical analysis of Conjecture~\ref{conjecture:rcgacomplexity-gonemax} is one of the future research works. Also, we would like to investigate the \rcga on a multi-valued \text{OneMax} problem where all the frequencies are restricted by a specific upper and lower border value. And, at the end of the experiments, we state another expected runtime in a conjecture. Further, from the experiments, we believe that our runtime bounds can be improved.

%
\bibliographystyle{splncs04}
\bibliography{References}

\end{document}